\documentclass[12pt]{article}

\usepackage[colorlinks]{hyperref}            
\usepackage{color}
\usepackage{graphicx,subfigure,amsmath,amssymb,amsfonts,bm,epsfig,epsf,url,dsfont}
\usepackage{times}
\usepackage{bbm}      
\usepackage{booktabs}
\usepackage{cases}
\usepackage{fullpage}
\usepackage[small,bf]{caption}
\usepackage[top=1in,bottom=1in,left=1in,right=1in]{geometry}

\newtheorem{theorem}{Theorem}
\newtheorem{definition}{Definition}
\newtheorem{lemma}{Lemma}
\newcommand{\BlackBox}{\rule{1.5ex}{1.5ex}}  
\newenvironment{proof}{\par\noindent{\bf Proof\ }}{\hfill\BlackBox\\[2mm]}

\newcommand{\conv}{Conv}
\newcommand{\Conv}{Conv}

\newcommand{\K}{\mathrm{KL}}
\newcommand{\tz}{\tilde{z}}

\renewcommand{\phi}{\varphi}

\newcommand{\IND}{\mathbbm{1}}

\renewcommand{\P}{\mathbb{P}}
\newcommand{\E}{\mathbb{E}}
\newcommand{\N}{\mathbb{N}}
\newcommand{\R}{\mathbb{R}}

\newcommand{\KL}{\mathrm{KL}}

\newcommand{\cD}{\mathcal{D}}
\newcommand{\cZ}{\mathcal{Z}}
\newcommand{\cA}{\mathcal{A}}
\newcommand{\cB}{\mathcal{B}}

\newcommand{\oD}{\overline{\mathcal{D}}}
\newcommand{\oR}{\overline{R}}

\def\ds1{\mathds{1}}
\renewcommand{\epsilon}{\varepsilon}

\newcommand\Var{{\dsV\text{ar}}\,}
\newcommand\dsV{\mathbb{V}} 
\newcommand\lam{\lambda}

\newcommand{\argmin}{\mathop{\mathrm{argmin}}}

\renewcommand{\tilde}{\widetilde}

\newlength{\minipagewidth}
\newcommand{\bookbox}[1]{
\par\medskip\noindent
\framebox[\textwidth]{
\begin{minipage}{\minipagewidth}
{#1}
\end{minipage} } \par\medskip }

\newcommand{\beq}{\begin{equation}}
\newcommand{\eeq}{\end{equation}}

\newcommand{\beqa}{\begin{eqnarray}}
\newcommand{\eeqa}{\end{eqnarray}}

\newcommand{\beqan}{\begin{eqnarray*}}
\newcommand{\eeqan}{\end{eqnarray*}}

\def\ba#1\ea{\begin{align*}#1\end{align*}} 
\def\banum#1\eanum{\begin{align}#1\end{align}} 

\begin{document}

\title{Regret in Online Combinatorial Optimization
}

\author{
Jean-Yves Audibert \\
Imagine, Universit{\'e} Paris Est, \\
and Sierra, CNRS/ENS/INRIA \\
{\tt audibert@imagine.enpc.fr} \\ \\
S{\'e}bastien Bubeck \\
Department of Operations Research and Financial Engineering, \\
Princeton University \\
{\tt sbubeck@princeton.edu} \\ \\
G{\'a}bor Lugosi \\
ICREA and Pompeu Fabra University \\
{\tt gabor.lugosi@upf.edu}  \\ \\
}

\date{\today}

\maketitle

\begin{abstract} 
We address online linear optimization problems when the possible
actions of the decision maker are represented by binary vectors.  The
regret of the decision maker is the difference between her realized
loss and the minimal loss she would have achieved by picking, in
hindsight, the best possible action.  Our goal is to understand the
magnitude of the best possible (minimax) regret.  We study the problem
under three different assumptions for the feedback the decision maker
receives: full information, and the partial information models of the
so-called ``semi-bandit'' and ``bandit'' problems.  In the full information
case we show that the standard
exponentially weighted average forecaster is a provably suboptimal strategy.
For the semi-bandit model, by combining the
Mirror Descent algorithm and the INF (Implicitely Normalized
Forecaster) strategy, we are able to prove the first optimal bounds. 
Finally, in the bandit case we discuss existing results
in light of a new lower bound, and suggest a conjecture on the optimal
regret in that case.
\end{abstract}

\section{Introduction.}

In this paper we consider the framework of online linear optimization.
The setup may be described as a repeated game between a ``decision maker''
(or simply ``player'' or ``forecaster'') and an ``adversary'' as follows:
at each time instance $t = 1, \ldots, n$, the player chooses, possibly in a randomized way, an action from a given finite action set $\cA \subset \R^d$. The action chosen by the player at time $t$ is denoted by $a_t \in \cA$. Simultaneously to the player, the adversary chooses a loss vector $z_t \in \cZ \subset \R^d$ and the loss incurred by the forecaster is $a_t^T z_t$. The goal of the player is to minimize the expected cumulative loss $\E \sum_{t=1}^n a_t^T z_t$ where the expectation is taken with respect to the player's internal randomization (and eventually the adversary's randomization). 

In the basic ``full-information'' version of this problem, the player observes the adversary's move $z_t$ at the end of round $t$. Another important model for feedback is the so-called {\em bandit} problem, in which the player only observes the incurred loss $a_t^T z_t$. As a measure of performance we define the regret \footnote{In the full
  information version, it is straightforward to obtain upper bounds
  for the stronger notion of regret 
   $\E \sum_{t=1}^n a_t^T z_t - \E \min_{a \in \cA}
  \sum_{t=1}^n a^T z_t $ which is always at least as large as
  $R_n$. However, for partial information games, this requires more
  work. In this paper we only consider $R_n$ as a measure of the regret.} 
of the player as
\[
R_n = \E \sum_{t=1}^n a_t^T z_t - \min_{a \in \cA} \E \sum_{t=1}^n a^T z_t~.
\] 
In this paper we address a specific example of online linear optimization: we assume that the action set $\cA$ is a subset of the $d$-dimensional hypercube $\{0,1\}^d$ such that $\forall a \in \cA, ||a||_1 = m$, and the adversary has a bounded loss per coordinate, that is\footnote{Note that since all actions have the same size, i.e. $||a||_1 = m, \forall a\in \cA$, one can reduce the case of  $\cZ = [\alpha, \beta]^d$ to $\cZ = [0,1]^d$ via a simple renormalization.}  $\cZ = [0,1]^d$. We call this setting {\em online combinatorial optimization}. As we will see below, this restriction of the general framework contains a rich class of problems. Indeed, in many interesting cases, actions are naturally represented by Boolean vectors.

In addition to the full information and bandit versions of online combinatorial optimization, we also consider another type of feedback which makes sense only in this combinatorial setting. In the {\em semi-bandit} version, we assume that the player observes only the coordinates of $z_t$ that were played in $a_t$, that is the player observes the vector $(a_t(1) z_t(1), \ldots, a_t(d) z_t(d))$. All three variants of online combinatorial optimization are sketched in Figure \ref{fig:1}.
More rigorously, online combinatorial optimization is defined as a repeated
game between a ``player'' and an ``adversary.'' At each round $t=1,\ldots,n$
of the game, the player chooses a probability distribution $p_t$ over
the set of actions $\cA\subset \{0,1\}^d$ and draws a random action $a_t\in \cA$
according to $p_t$. Simultaneously, the adversary chooses a vector
$z_t\in [0,1]^d$. 
More formally, $z_t$ is a measurable function of the ``past'' 
$(p_s,a_s,z_s)_{s=1,\ldots,t-1}$. In the full information case, 
$p_t$ is a measurable function of $(p_s,a_s,z_s)_{s=1,\ldots,t-1}$. 
In the semi-bandit case, $p_t$ is a measurable function of 
$(p_s,a_s,(a_s(i)z_s(i))_{i=1,\ldots,d})_{s=1,\ldots,t-1}$ and in the
bandit problem it is a measurable function of
$(p_s,a_s,(a_s^Tz_s))_{s=1,\ldots,t-1}$. 

\begin{figure}[t]
\bookbox{\small
{Parameters:} set of actions $\cA \subset \{0,1\}^d$; number of rounds $n \in \N$.

\medskip\noindent
For each round $t=1,2,\ldots,n$;
\begin{itemize}
\item[(1)]
the player chooses a probability distribution $p_t$ over $\cA$ and draws a random action $a_t \in \cA$ according to $p_t$;
\item[(2)]
\mbox{simultaneously, the adversary selects a loss vector $z_t \in [0,1]^d$ (without revealing it)};
\item[(3)]
the player incurs the loss $a_t^T z_t $. She observes
\begin{itemize}
\item the loss vector $z_t$ in the full information setting,
\item the coordinates $z_t(i) a_t(i)$ in the semi-bandit setting,
\item the instantaneous loss $a_t^T z_t $ in the bandit setting.
\end{itemize}
\end{itemize}

\medskip\noindent
{\em Goal:} The player tries to minimize her cumulative loss $\sum_{t=1}^n a_t^T z_t$.
}
\caption{Online combinatorial optimization.}
\label{fig:1}
\end{figure}

\subsection{Motivating examples.} \label{sec:ex}
Many problems can be tackled under the online combinatorial optimization framework. We give here three simple examples:
\begin{itemize}
\item \textbf{$\mathbf{m}$-sets.} In this example we consider the set $\cA$ of all 
$\binom{d}{m}$
Boolean vectors in dimension $d$ with exactly $m$ ones. In other words, at every time step, the player selects $m$ actions out of $d$ possibilities. When $m=1$, the semi-bandit and bandit versions coincide and correspond to the standard (adversarial) multi-armed bandit problem.
\item \textbf{Online shortest path problem.} Consider a communication network represented by a graph in which one has to send a sequence of packets from one fixed vertex to another.
For each packet one chooses a path 
through the graph and suffers a certain delay which is the sum of the delays on the edges of the path. Depending on the traffic,
the delays on the edges may change, and, at the end of each round, according to the assumed level of feedback, the player 
observes either the delays of all edges, the delays of each edge on the chosen path, or only the total delay of the chosen
path. The player's objective is to minimize the total delay for the sequence of packets. 

One can represent the set of valid paths from the starting vertex to the end vertex as a set $\cA \subset \{0,1\}^d$ where $d$ is the number 
of edges. If at time $t$, $z_t \in [0,1]^d$ is the vector of delays on the edges, then the delay of a path $a \in \cA$ is 
$z_t^T a$. Thus this problem is an instance of online combinatorial optimization
in dimension $d$, where $d$ is the number of edges in the graph. In this paper
we assume, for simplicity, that all valid paths have the same length $m$.
\item \textbf{Ranking.} Consider the problem of selecting a ranking of $m$ items out of $M$ possible items. For example a website could have a set of $M$ ads, and it has to select a ranked list of $m$ of these ads to appear on the webpage. One can rephrase this problem as selecting a matching of size $m$ on the complete bipartite graph $K_{m,M}$ 
(with $d=m \times M$ edges). In the online learning version of this problem, each day the website chooses one such list, and gains one dollar for each click on the ads. This problem can easily be formulated as an online combinatorial optimization problem.
\end{itemize}
Our theory applies to many more examples, such as spanning trees (which can be useful in certain communication problems), or $m$-intervals.

\subsection{Previous work.} \label{sec:previousworks}
\begin{itemize}
\item{\textbf{Full Information.}  The full-information setting is now fairly well understood, and an optimal regret bound (in terms of $m, d, n$) was obtained by Koolen, Warmuth, and Kivinen \cite{KWK10}. 
Previous papers under full information feedback also include 
Gentile and Warmuth \cite{GW98},
Kivinen and Warmuth \cite{KW01}, 
Grove, Littlestone, and Schuurmans \cite{GLS01}, 
Takimoto and Warmuth \cite{TW03}, 
Kalai and Vempala \cite{KV05}, 
Warmuth and Kuzmin \cite{WK08}, 
Herbster and Warmuth \cite{HW09},
and Hazan, Kale, and Warmuth \cite{HKW10}.}
\item{\textbf{Semi-bandit.} The first paper on the adversarial multi-armed bandit problem (i.e., the special case of $m$-sets with $m=1$) is 
by Auer, Cesa-Bianchi, Freund, and Schapire \cite{ACFS03} who derived a regret bound of order $\sqrt{d n \log d}$. This result was improved to $\sqrt{d n}$ by Audibert and Bubeck \cite{AB09, AB10}. 
Gy\"orgy, Linder, Lugosi, and Ottucs\'ak \cite{GLLO07} consider the online shortest path problem and derive suboptimal regret bounds (in terms of the dependency on $m$ and $d$). Uchiya, Nakamura, and Kudo \cite{UNK10} (respectively Kale, Reyzin, and Schapire \cite{KRS10}) derived optimal regret bounds for the case of $m$-sets (respectively for the problem of ranking selection) up to logarithmic factors.}
\item{\textbf{Bandit.} McMahan and Blum \cite{MB04}, and Awerbuch and Kleinberg \cite{AK04} were the first to consider this setting, and obtained suboptimal regret bounds (in terms of $n$). The first paper with optimal dependency in $n$ was by Dani, Hayes, and Kakade \cite{DHK08}. The dependency on $m$ and $d$ was then improved in various ways by Abernethy, Hazan, and Rakhlin \cite{AHR08}, Cesa-Bianchi, and Lugosi \cite{CL11}, and Bubeck, Cesa-Bianchi, and Kakade \cite{BCK12}. We discuss these bounds in detail in Section \ref{sec:bandit}. In particular, we argue that the optimal regret bound in terms of $d$ (and $m$) is still an open problem.}
\end{itemize}
We also refer the interested reader to the recent survey \cite{BC12} for an overview of bandit problems in various other settings.

\subsection{Contribution and contents of the paper.}
In this paper we are primarily interested in the optimal {\em minimax regret} in terms of $m, d$ and $n$. More precisely, our aim is to determine the order of magnitude of the following quantity: For a given feedback assumption, write $\sup$ for the supremum over all adversaries and $\inf$ for the infimum over all allowed strategies for the player under the feedback assumption. 
(Recall the definition of ``adversary'' and ``player'' from the introduction.) 
Then we are interested in
$$\max_{\cA \subset \{0,1\}^d: \forall a \in \cA, ||a||_1 = m} \inf \sup R_n .$$


Our contribution to the study of this quantity is threefold. First, we unify the algorithms used in 
Abernethy, Hazan, and Rakhlin \cite{AHR08}, 
Koolen, Warmuth, and Kivinen \cite{KWK10}, 
Uchiya, Nakamura, and Kudo \cite{UNK10}, 
and Kale, Reyzin, and Schapire \cite{KRS10} under the umbrella of mirror descent. The idea of mirror descent goes back to 
Nemirovski \cite{Nem79}, 
Nemirovski and Yudin \cite{NY83}. 
A somewhat similar concept was re-discovered in online learning by
Herbster and Warmuth \cite{HW98}, 
Grove, Littlestone, and Schuurmans \cite{GLS01},  
Kivinen and Warmuth \cite{KW01} under the name of potential-based gradient descent, see \cite[Chapter 11]{CL06}. Recently, these ideas have been flourishing, see for instance 
Shalev-Schwartz \cite{Sha07}, 
Rakhlin \cite{Rak09}, 
Hazan \cite{Haz11}, 
and Bubeck \cite{Bub11}.  Our main theorem (Theorem \ref{th:OSMD}) allows one to recover almost all known regret bounds for online combinatorial optimization. This first contribution leads to our second main
result, the improvement of the known upper bounds for the semi-bandit game. In particular, we propose a different proof of the minimax regret bound of the order of $\sqrt{n d}$ in the standard $d$-armed bandit game that is much simpler than the one provided in Audibert and Bubeck \cite{AB10} (which also improves the constant factor). In addition to these upper bounds we prove two new lower bounds. First we answer a question of Koolen, Warmuth, and Kivinen \cite{KWK10} by showing that the exponentially weighted average forecaster is provably suboptimal for online combinatorial optimization. Our second lower bound is a minimax lower bound in the bandit setting which improves known results by an order of magnitude. A summary of known bounds and the new bounds proved in this paper can be found
in Table \ref{table:1}.

The paper is organized as follows. In Section \ref{sec:alg} we introduce the two algorithms discussed in this paper. In particular in Section \ref{sec:Exp2} we discuss the popular exponentially weighted average forecaster and we show that it is a provably suboptimal strategy. Then in Section \ref{sec:OSMD} we describe our main algorithm, {\sc osmd} (Online Stochastic Mirror Descent), and prove a general regret bound in terms of the Bregman divergence of the Fenchel-Legendre dual of the 
Legendre function defining the strategy.
In Section \ref{sec:semibandit} we derive upper bounds for the regret in the semi-bandit case for {\sc osmd} with appropriately chosen Legendre functions. Finally in Section \ref{sec:bandit} we prove a new lower bound for the bandit setting, and we formulate a conjecture on the correct order of magnitude of the regret for that problem based on this new result and the regret bounds obtained in \cite{AHR08, BCK12}.

\begin{table}[t]
\begin{center}
\begin{tabular}{c|c|c|c}
 & { Full Information} & { Semi-Bandit} & {Bandit} 
\\  \hline & & & \\
  {Lower Bound} & $m \sqrt{n \log \frac{d}{m}}$ & $\sqrt{m d n}$ & $\mathbf{m \sqrt{d n}}$
\\  \hline & & & \\
  {Upper Bound} & $m \sqrt{n \log \frac{d}{m}}$ & $\mathbf{\sqrt{m d n}}$ & $m^{3/2} \sqrt{d n \log\frac{d}{m}}$
  \end{tabular}
\end{center}
\caption{Bounds on the minimax regret (up to constant factors). 
The new results are set in boldface. In this paper we also show that {\sc exp2} in the full information case has a regret bounded below by $d^{3/2} \sqrt{n}$ (when $m$ is of order $d$).}
\label{table:1}
\end{table}


\section{Algorithms.} \label{sec:alg}
In this section we discuss two classes of algorithms that have been proposed for online combinatorial optimization.

\subsection{Expanded Exponential weights ({\sc exp2}).} \label{sec:Exp2}
The simplest approach to online combinatorial optimization is to consider
each action of $\cA$ as an independent ``expert,'' and then apply a generic  regret minimizing strategy. Perhaps the most popular such strategy
is the exponentially weighted average forecaster (see, e.g., \cite{CL06}).
(This strategy is sometimes called Hedge, see Freund and Schapire \cite{FS97}.) 
 We call the resulting
strategy for the online combinatorial optimization problem {\sc exp2}, see Figure
\ref{fig:Exp2}. In the full information setting, {\sc exp2} corresponds
to ``Expanded Hedge,'' as defined in Koolen, Warmuth, and Kivinen \cite{KWK10}. In the
semi-bandit case, {\sc exp2} was studied by Gy\"orgy, Linder, Lugosi, and Ottucs\'ak \cite{GLLO07}
while in the bandit case in 
Dani, Hayes, and Kakade \cite{DHK08},
Cesa-Bianchi and Lugosi \cite{CL11}, and
Bubeck, Cesa-Bianchi, and Kakade  \cite{BCK12}. 
Note that in the bandit case, {\sc exp2} is mixed with an {\em exploration distribution}, see Section \ref{sec:bandit} for more details.

Despite strong interest in this strategy, no optimal regret bound has been derived for it in the combinatorial setting. More precisely, the best bound (which can be derived from a standard argument, see for example \cite{DHK08} or \cite{KWK10}) is of order $m^{3/2} \sqrt{n \log \left( \frac{d}{m} \right)}$. On the other hand, in \cite{KWK10} the authors showed that by using Mirror Descent (see next section) with the negative entropy, one obtains a regret bounded by $m \sqrt{n \log \left( \frac{d}{m} \right)}$. Furthermore this latter bound is clearly optimal (up to a numerical constant), as one can see from the standard lower bound in prediction with expert advice (consider the set $\cA$ that corresponds to playing $m$ expert problems in parallel with $d/m$ experts in each problem). In \cite{KWK10} the authors leave as an open question the problem of whether it would be possible to improve the bound for {\sc exp2} to obtain the optimal order of magnitude. The following theorem shows that this is impossible, and that in fact {\sc exp2} is a provably suboptimal strategy.

\begin{theorem} \label{lb:expinfty}
Let $n \geq d$. There exists a subset $\cA \subset \{0,1\}^d$ such that in the full information setting, the regret of the {\sc exp2} strategy (for any learning rate $\eta$), satisfies
$$\sup_{\text{adversary}} R_n \geq 0.01 \, d^{3/2} \sqrt{n} .$$
\end{theorem}

The proof is deferred to the Appendix.

\begin{figure}[t]
\bookbox{
{\em {\sc exp2}:}

\medskip\noindent

Parameter: Learning rate $\eta$.

\medskip\noindent
Let $p_1=\big(\frac1{|\cA|},\ldots,\frac1{|\cA|}\big) \in \R^{|\cA|}$.

\medskip\noindent
For each round $t=1,2,\ldots,n$;
\begin{itemize}
\item[(a)]
Play $a_t \sim p_t$ and observe 
\begin{itemize}
\item the loss vector $z_t$ in the full information game,
\item the coordinates $z_t(i) \IND_{a_t(i) = 1}$ in the semi-bandit game,
\item the instantaneous loss $a_t^T z_t $ in the bandit game.
\end{itemize}
\item[(b)]
Estimate the loss vector $z_t$ by $\tz_t$.
For instance, one may take
\begin{itemize}
\item $\tz_t=z_t$ in the full information game,
\item $\tilde{z}_t(i) = \frac{z_t(i)}{\sum_{a \in\cA:a(i)=1} p_t(a)} a_t(i)$ in the semi-bandit game,
\item $\tilde{z_t} = P_t^+ a_t a_t^T z_t,$ 
with $P_t = \E_{a \sim p_t} (a a^T)$ in the bandit game.
\end{itemize}
\item[(c)]
Update the probabilities, for all $a \in \cA$, 
$$p_{t+1}(a) =\frac{\exp(- \eta a^T \tz_t) p_t(a)}{\sum_{b \in\cA} \exp(- \eta b^T \tz_t^T) p_t(b)}.$$ 
\end{itemize}
}
\caption{The {\sc exp2} strategy. 
The notation $\E_{a \sim p_t}$ denotes expected value with respect to the
random choice of $a$ when it is distributed according to $p_t$.
}\label{fig:Exp2}
\end{figure}

\subsection{Online Stochastic Mirror Descent.} \label{sec:OSMD}
In this section we describe the main algorithm studied in this paper. We call it Online Stochastic Mirror Descent ({\sc osmd}). Each term in this name refers to a part of the algorithm: {\em Mirror Descent} originates in the work of Nemirovski and Yudin \cite{NY83}. The idea of mirror descent is to perform a gradient descent, where the update with the gradient is performed in the dual space (defined by some Legendre function $F$) rather than in the primal (see below for a precise formulation). The {\em Stochastic} part takes its origin from Robbins and Monro \cite{RM51} and from Kiefer and Wolfowitz \cite{KW52}. The key idea is that it is enough to observe an unbiased estimate of the gradient rather than the true gradient in order to perform a gradient descent. Finally the {\em Online} part comes from Zinkevich \cite{Zin03}. Zinkevich derived the Online Gradient Descent ({\sc ogd}) algorithm, which is a version of gradient descent tailored to online optimization.

To properly describe the {\sc osmd} strategy, we recall a few concepts from convex analysis, see Hiriart-Urruty and Lemar\'echal \cite{HL01} for a thorough treatment of this subject. Let $\cD \subset \R^d$ be an open convex set, and $\oD$ the closure of $\cD$.

\begin{definition}
We call Legendre any continuous function $F:\oD\rightarrow\R$ such that 
\begin{itemize}
\item[(i)] $F$ is strictly convex continuously differentiable
on $\cD$,
\item[(ii)] $\lim_{x \rightarrow \oD \setminus \cD} ||\nabla F(x)|| = +\infty.$\footnote{By the equivalence of norms in $\R^d$, this definition does not depend on the choice of the norm.}
\end{itemize}
The Bregman divergence $D_F: \oD\times \cD$ associated to a Legendre function $F$ is defined by
  $$
  D_F(x,y) = F(x) - F(y) - (x-y)^T\nabla F(y).
  $$ 
Moreover, we say that $\cD^* = \nabla F (\cD)$ is the dual space of $\cD$ under $F$. We also denote by $F^*$ the Legendre-Fenchel transform of $F$ defined by
$$F^*(u) = \sup_{x \in \oD} \left(x^T u - F(x)\right)~.$$
\end{definition}

\begin{lemma}
Let $F$ be a Legendre function. Then $F^{**} = F$ and $\nabla F^* = (\nabla F)^{-1}$ on the set $\cD^*$. Moreover, $\forall x, y \in \cD$,
\begin{equation} \label{eq:transrelation}
D_F(x,y) = D_{F^*}(\nabla F(y), \nabla F(x)) .
\end{equation}
\end{lemma}
The lemma  above is the key to understanding how a Legendre function acts on the space. The gradient $\nabla F$ maps $\cD$ to the dual space $\cD^*$, and $\nabla F^*$ is the inverse mapping from the dual space to the original (primal) space. Moreover, \eqref{eq:transrelation} shows that the Bregman divergence in the primal space corresponds exactly to the Bregman divergence of the Legendre-Fenchel transform in the dual space. A proof of this result can be found, for example, in [Chapter 11, \cite{CL06}].

We now have all ingredients to describe the {\sc osmd} strategy, see Figure \ref{fig:OSMD} for the precise formulation. Note that step (d) is well defined if the following consistency condition is satisfied:
\begin{equation} \label{eq:consistency}
\nabla F(x) - \eta \tilde{z}_t  \in \cD^*, \forall x \in \Conv(\cA) \cap \cD. 
\end{equation}

In the full information setting, algorithms of this type were studied by Abernethy, Hazan, and Rakhlin \cite{AHR08}, Rakhlin \cite{Rak09}, and Hazan \cite{Haz11}. In these papers the authors adopted the presentation suggested by Beck and Teboulle \cite{BT03}, which corresponds to a Follow-the-Regularized-Leader ({\sc ftrl}) type strategy. There the focus was on $F$ being strongly convex with respect to some norm. Moreover, in \cite{AHR08} the authors also consider the bandit case, and switch to $F$ being a self-concordant barrier for the convex hull of $\cA$ (see Section \ref{sec:bandit} for more details). Another line of work studied this type of algorithms with $F$ being the negative entropy, see Koolen, Warmuth, and Kivinen \cite{KWK10} for the full information case and Uchiya, Nakamura, and Kudo \cite{UNK10}, Kale, Reyzin, and Schapire \cite{KRS10} for specific instances of the semi-bandit case. All these results are unified and described in details in Bubeck \cite{Bub11}. In this paper we consider a new type of Legendre functions $F$ inspired by Audibert and Bubeck \cite{AB10}, see Section \ref{sec:semibandit}.

Regarding computational complexity, {\sc osmd} is efficient as soon as the polytope $\Conv(\cA)$ can be described by a polynomial (in $d$) number of constraints. Indeed in that case steps (a)-(b) can be performed efficiently jointly (one can get an algorithm by looking at the proof of Carath{\'e}odory's theorem), and step (d) is a convex program with a polynomial number of constraints. In many interesting examples (such as $m$-sets, selection of rankings, spanning trees, paths in acyclic graphs) one can describe the convex hull of $\cA$ by a polynomial number of constraints, see Schrijver \cite{Sch03}. On the other hand, there also exist important examples where this is not the case (such as paths on general graphs). Also note that for some specific examples it is
possible to implement {\sc osmd} with improved computational complexity, see Koolen, Warmuth, and Kivinen \cite{KWK10}.

In this paper we restrict our attention to the combinatorial learning
setting in which $\cA$ is a subset of $\{0,1\}^d$ and the loss is linear. However, one should
note that this specific form of $\cA$ plays no role in the definition
of {\sc osmd}. Moreover, if the loss is not linear, then one can modify {\sc osmd} by performing a gradient update with a gradient of the loss (rather than the loss vector $z_t$). See Bubeck \cite{Bub11} for more details on this approach.

\begin{figure}[t]
\bookbox{{\em {\sc osmd}:}

\medskip\noindent

Parameters: 
\begin{itemize}
\item learning rate $\eta > 0$,
\item Legendre function $F$ defined on $\oD \supset \Conv(\cA)$.
\end{itemize}

\medskip\noindent
Let $x_1 \in \argmin_{x \in \Conv(\cA)} F(x)$.

\medskip\noindent
For each round $t=1,2,\ldots,n$;
\begin{itemize}
\item[(a)]
Let $p_t$ be a distribution on the set $\cA$ such that $x_t = \E_{a \sim p_t} a$.
\item[(b)] 
Draw a random action $a_t$ according to the distribution $p_t$ and observe the feedback.
\item[(c)]
Based on the observed feedback, estimate the loss vector $z_t$ by $\tz_t$.
\item[(d)]
Let $w_{t+1}\in \cD$ satisfy
  \begin{align} \label{eq:wp}
  \nabla F(w_{t+1}) = \nabla F(x_t) - \eta \tilde{z}_t.
  \end{align} 
\item[(e)]
Project the weight vector $w_{t+1}$ defined by \eqref{eq:wp} on the convex hull of $\cA$:
\begin{equation} \label{eq:proj}
x_{t+1} \in \argmin_{x \in \Conv(\cA)} D_F(x,w_{t+1}).
\end{equation}
\end{itemize}
}
\caption{Online Stochastic Mirror Descent ({\sc OSMD}).}\label{fig:OSMD}
\end{figure}

The following result is at the basis of our improved regret bounds for {\sc osmd}
in the semi-bandit setting, see Section \ref{sec:semibandit}.

\begin{theorem} \label{th:OSMD}
Suppose that \eqref{eq:consistency} is satisfied and the 
loss estimates are unbiased in the sense that
$\E_{a_t \sim p_t} \tz_t=z_t$. Then the regret
of the {\sc osmd} strategy satisfies
$$R_n \leq \frac{\sup_{a \in \cA} F(a) - F(x_1)}{\eta} + \frac{1}{\eta} \sum_{t=1}^n \E D_{F^*}\bigg(\nabla F(x_t) - \eta \tilde{z}_t, \nabla F(x_t)\bigg) .$$
\end{theorem}

\begin{proof}
Let $a \in \cA$. Using that $a_t$ and $\tilde{z}_t$ are unbiased estimates of $x_t$ and $z_t$, we have
$$\E \sum_{t=1}^n (a_t - a)^T z_t = \E \sum_{t=1}^n (x_t - a)^T \tz_t .$$
Using \eqref{eq:wp}, and applying the definition of the Bregman divergences, one obtains
  \begin{align*}
  \eta \tz_t^T (x_t - a) & = (a-x_t)^T\big(\nabla F(w_{t+1}) - \nabla F(x_t) \big)\\
  & = D_F(a,x_t)+D_F(x_t,w_{t+1})-D_F(a,w_{t+1}).
  \end{align*}
By the Pythagorean theorem for Bregman divergences (see, e.g., Lemma 11.3 of \cite{CL06}), we have 
  $D_F(a,w_{t+1}) \ge D_F(a,x_{t+1}) + D_F(x_{t+1},w_{t+1}),$ hence
  \begin{align*}
  \eta \tz_t^T (x_t - a) \le D_F(a,x_t)+D_F(x_t,w_{t+1})-D_F(a,x_{t+1})-D_F(x_{t+1},w_{t+1})~.
  \end{align*}
Summing over $t$ gives  
\[
  \sum_{t=1}^n \eta \tz_t^T (x_t - a)  \le  D_F(a,a_1)-D_F(a,a_{n+1}) 
   +\sum_{t=1}^n \big(D_F(x_t,w_{t+1}) -D_F(x_{t+1},w_{t+1})\big)~.
\]
By the nonnegativity of the Bregman divergences, we get 
$$\sum_{t=1}^n \eta \tz_t^T (x_t - a)
\le D_F(a,a_1)+\sum_{t=1}^n D_F(x_t,w_{t+1}).$$
From \eqref{eq:transrelation}, one has
  $
  D_F(x_t,w_{t+1})=D_{F^*}\big(\nabla F(x_t)- \eta \tilde{z}_t,\nabla F(x_t)\big).
  $
Moreover, by writing the first-order optimality condition for $x_1$, one directly obtains $D_F(a,x_1) \leq F(a) - F(x_1)$ which concludes the proof.
\end{proof}

Note that, if $F$ admits an Hessian,
denoted $\nabla^2F$, that is always invertible, then one can prove that,
up to a third-order term \big(in $\tz_t$\big), the regret bound can be
written as
\begin{equation} \label{eq:intuition}
R_n \lessapprox \frac{\sup_{a \in \cA} F(a) - F(x_1)}{\eta}  + \frac{\eta}{2} \sum_{t=1}^n \tz_t^T \left( \nabla^2 F(x_t) \right)^{-1} \tz_t .
\end{equation}
The main technical difficulty is to control the third-order error term in this inequality.

\section{Semi-bandit feedback.} \label{sec:semibandit}
In this section we consider online combinatorial optimization with semi-bandit feedback. 
As we already discussed, in the full information case Koolen, Warmuth, and Kivinen \cite{KWK10} proved that {\sc osmd} with the negative entropy is a minimax optimal strategy.
We first prove a regret bound when one uses this strategy with the following estimate for the loss vector:
\begin{equation} \label{eq:semibanditestimate}
\tz_t(i) = \frac{z_t(i) a_t(i)}{x_t(i)}.
\end{equation}
Note that this is a valid estimate since it makes only use of $(z_t(1) a_t(1), \ldots, z_t(d) a_t(d))$. Moreover, it is unbiased with respect to the random draw of $a_t$ from $p_t$, since by definition, $\E_{a_t \sim p_t} a_t(i) = x_t(i)$. In other words, $\E_{a_t \sim p_t} \tz_t(i) = z_t(i)$.

\begin{theorem} \label{th:negentropy}
The regret of OSMD with $F(x) = \sum_{i=1}^d x_i \log x_i - \sum_{i=1}^d x_i$ (and $\cD = (0, +\infty)^d$) and any non-negative unbiased loss estimate $\tz_t(i) \geq 0$ satisfies
$$R_n \leq \frac{m \log \frac{d}{m}}{\eta} + \frac{\eta}{2} \sum_{t=1}^n \sum_{i=1}^d x_t(i) \tilde{z}_t(i)^2 .$$
In particular, with the estimate \eqref{eq:semibanditestimate} and $\eta= \sqrt{2 \frac{m \log{d}{m}}{n d}}$,
$$R_n \leq  \sqrt{2 m d n \log \frac{d}{m}} .$$
\end{theorem}

\begin{proof}
One can easily see that for the negative entropy the dual space is $\cD^*=\R^d$. Thus, \eqref{eq:consistency} is verified and {\sc osmd} is well defined. Moreover, again by straightforward computations, one can also see that
\begin{equation} \label{eq:bregentropy}
D_{F^*}\bigg(\nabla F(x), \nabla F(y)\bigg) =\sum_{i=1}^d y(i) \ \Theta\bigg( (\nabla F(x) - \nabla F(y))(i) \bigg)~,
\end{equation}
where $\Theta(x)= \exp(x) - 1 - x$.
Thus, using Theorem \ref{th:OSMD} and the facts that $\Theta(x) \leq \frac{x^2}{2}$ for $x \leq 0$ and $\sum_{i=1}^d x_t(i) \leq m$, one obtains
\begin{eqnarray} 
R_n & \leq & \frac{\sup_{a \in \cA} F(a) - F(x_1)}{\eta} + \frac{1}{\eta} \sum_{t=1}^n \E D_{F^*}\bigg(\nabla F(x_t) - \eta \tilde{z}_t, \nabla F(x_t)\bigg) \notag \\
& \leq & \frac{\sup_{a \in \cA} F(a) - F(x_1)}{\eta} + \frac{\eta}{2} \sum_{t=1}^n \sum_{i=1}^d x_t(i) \tilde{z}_t(i)^2 \notag
\end{eqnarray}
The proof of the first inequality is concluded by noting that:
\begin{equation} \label{eq:linexp2}
F(a) - F(x_1) \leq \sum_{i=1}^d x_1(i) \log \frac{1}{x_1(i)} \leq m \log \left( \sum_{i=1}^d \frac{x_1(i)}{m} \frac{1}{x_1(i)}\right) = m \log \frac{d}{m}. \notag
\end{equation}
The second inequality follows from
$$\E x_t(i) \tilde{z}_t(i)^2 \leq \E \frac{a_t(i)}{x_t(i)} = 1.$$
\end{proof}

Using the standard $\sqrt{d n}$ lower bound for the multi-armed bandit (which corresponds to the case where $\cA$ is the canonical basis), see e.g., [Theorem 30, \cite{AB10}], one can directly obtain a lower bound of order $\sqrt{m d n}$ for our setting. Thus the upper bound derived in Theorem \ref{th:negentropy} has an extraneous logarithmic factor compared to the lower bound.
This phenomenon already appeared in the basic multi-armed bandit setting. In that case, the extra logarithmic factor was removed in Audibert and Bubeck \cite{AB09} by resorting to a new class of strategies for the expert problem, called INF (Implicitely Normalized Forecaster). Next we generalize this class of algorithms to the combinatorial setting, and thus remove the extra logarithmic factor. First we introduce the notion of a potential and the associated Legendre function.
\begin{definition}
Let $\omega \ge 0$.
A function 
$\psi: (-\infty,a) \rightarrow \R^*_+$ for some 
$a\in\R\cup\{+\infty\}$ is called an 
$\omega$-potential if it is convex,
 continuously differentiable, and satisfies
\begin{align*} 
& \lim_{x\rightarrow -\infty} \psi(x)=\omega~, &&
\lim_{x\rightarrow a} \psi(x)= +\infty~, \notag \\
& \psi' > 0~, && \int_{\omega}^{\omega+1} |\psi^{-1}(s)|ds <+\infty~.
\end{align*}
For every potential $\psi$ we associate the function $F_{\psi}$ defined on $\cD=(\omega, +\infty)^d$ by:
$$F_{\psi}(x) = \sum_{i=1}^d \int_{\omega}^{x_i} \psi^{-1}(s) ds .$$
\end{definition}

In this paper we restrict our attention to $0$-potentials which we will simply call {\em potentials}.
 A non-zero value of $\omega$ may be used to derive regret bounds that hold with high probability (instead of pseudo-regret bounds, see footnote 1).

The first order optimality condition for \eqref{eq:proj} implies that {\sc osmd} with $F_{\psi}$ is a direct generalization of INF with potential $\psi$, in the sense that the two algorithms coincide when $\cA$ is the canonical basis. Note, in particular, that with $\psi(x)=\exp(x)$ we recover the negative entropy for $F_{\psi}$. In \cite{AB10}, the choice of $\psi(x) = (-x)^q$ with $q >1$ was recommended. We show in Theorem \ref{th:osmd} that here, again, this choice gives a minimax optimal strategy.
\begin{lemma} \label{lem:psi}
Let $\psi$ be a potential. Then $F=F_{\psi}$ is Legendre
and for all $u,v \in \cD^*=(-\infty,a)^d$ such that $u_i \leq v_i, \forall i \in \{1,\ldots,d\}$,
$$D_{F^*}(u,v) \leq \frac{1}{2} \sum_{i=1}^d \psi'(v_i) (u_i - v_i)^2 .$$
\end{lemma}

\begin{proof}
A direct examination shows that $F=F_{\psi}$ is a Legendre function. Moreover, since
	$
  \nabla F^*(u)= (\nabla F)^{-1}(u)= \big(\psi(u_1),\dots,\psi(u_d)\big),
	$
we obtain
	$$
  D_{F^*}(u,v)=\sum_{i=1}^d \bigg(\int_{v_i}^{u_i} \psi(s)ds-(u_i-v_i)\psi(v_i)\bigg).
  $$
From a Taylor expansion, we get
	$$
	D_{F^*}(u,v)\le\sum_{i=1}^d \max_{s\in[u_i,v_i]} \frac12 \psi'(s) (u_i-v_i)^2.
	$$
Since the function $\psi$ is convex, and $u_i \leq v_i$, we have
	$$
	\max_{s\in[u_i,v_i]} \psi'(s) \le \psi'\big(\max(u_i,v_i)\big) \leq \psi'(v_i),
	$$
which gives the desired result.
\end{proof}

\begin{theorem}
\label{th:osmd}
Let $\psi$ be a potential. The regret of {\sc osmd} with $F=F_{\psi}$ and any non-negative unbiased loss estimate $\tilde{z}_t$ satisfies
$$R_n \leq \frac{\sup_{a \in \cA} F(a) - F(x_1)}{\eta} + \frac{\eta}{2} \sum_{t=1}^n \sum_{i=1}^d  \E \frac{\tilde{z}_t(i)^2}{(\psi^{-1})'(x_t(i))} .$$
In particular, with the estimate \eqref{eq:semibanditestimate}, $\psi(x) = (- x)^{-q}$, $q>1$,and $\eta= \sqrt{\frac{2}{q-1} \frac{m^{1 - 2/q}}{d^{1 - 2/q}} \frac{1}{n}}$,
$$R_n \leq q \sqrt{\frac{2}{q-1} m d n}~.$$
With $q=2$ this gives
$$R_n \leq 2 \sqrt{2 m d n}~.$$ 
\end{theorem}

In the case $m=1$, the above theorem improves the bound $R_n \leq 8 \sqrt{n d}$ obtained in Theorem 11 of \cite{AB10}.

\begin{proof}
First note that since $\cD^*=(-\infty,a)^d$ and $\tilde{z}_t$ has non-negative coordinates, {\sc osmd} is well defined (that is, \eqref{eq:consistency} is satisfied).

The first inequality follows from Theorem \ref{th:OSMD} and the fact that $\psi'(\psi^{-1}(s)) = \frac{1}{(\psi^{-1})'(s)}$.

Let $\psi(x) = (- x)^{-q}$. Then $\psi^{-1}(x) = - x^{-1/q}$ and $F(x)= - \frac{q}{q-1} \sum_{i=1}^d x_i^{1 - 1/q}$. In particular, note that by H{\"o}lder's inequality, since $\sum_{i=1}^d x_1(i) =m$,
$$F(a) - F(x_1) \leq \frac{q}{q-1} \sum_{i=1}^d x_1(i)^{1 - 1/q} \leq \frac{q}{q-1} m^{(q-1)/q} d^{1/q} .$$
Moreover, note that $(\psi^{-1})'(x) = \frac1q x^{- 1 - 1/q}$, and
$$\sum_{i=1}^d \E \frac{\tilde{z}_t(i)^2}{(\psi^{-1})'(x_t(i))} \leq q \sum_{i=1}^d x_t(i)^{1/q} \leq q m^{1/q} d^{1 - 1/q},$$
which concludes the proof.
\end{proof}

\section{Bandit feedback.} \label{sec:bandit}
In this section we consider online combinatorial optimization with bandit feedback. This setting is much more challenging than the semi-bandit case, and in order to obtain sublinear regret bounds all known strategies add an {\em exploration} component to the algorithm. For example, in {\sc exp2}, instead of playing an action at random according to the exponentially weighted average distribution $p_t$, one draws a random action from $p_t$ with probability $1-\gamma$ and from some fixed ``exploration'' distribution $\mu$ with probability $\gamma$. On the other hand, in {\sc osmd}, one randomly perturbs $x_t$ to some $\tilde{x}_t$, and then plays at random a point in $\cA$ such that on average one plays $\tilde{x}_t$. 

In Bubeck, Cesa-Bianchi, and Kakade \cite{BCK12}, the authors study the {\sc exp2} strategy with the exploration distribution $\mu$ supported on the contact points between the polytope $\conv(\cA)$ and the John ellipsoid of this polytope (i.e., the ellipsoid of minimal volume enclosing the polytope). Using this method they are able to prove the best known upper bound for online combinatorial optimization with bandit feedback. They show that the regret of {\sc exp2} mixed with  John's exploration (and with the estimate described in Figure \ref{fig:Exp2}) satisfies
$$R_n \leq 2 m^{3/2} \sqrt{3 d n \log \frac{e d}{m}} .$$
Our next theorem shows that no strategy can achieve a regret less than
a constant times $m\sqrt{dn}$, leaving a gap of a factor
of $\sqrt{m \log \frac{d}{m}}$. As we argue below, we conjecture that the lower bound is of the correct order of magnitude. However, improving the upper bound seems to require some substantially new ideas. Note that the following bound gives limitations that no strategy can surpass, on the contrary to Theorem \ref{lb:expinfty} which was dedicated to the {\sc exp2} strategy.

\begin{theorem}
\label{thm:minimaxlower}
Let $n \geq d \geq 2 m$. There exists a subset $\cA \subset \{0,1\}^d$ such that $||a||_1 = m, \forall a \in \cA$,
under bandit feedback, one has
\begin{equation} \label{eq:LBbanditcombi}
\inf_{\text{strategies}}\sup_{\text{adversaries}} R_n \geq 0.02 m \sqrt{d n}~,
\end{equation}
where the infimum and the supremum are taken over the class of strategies
for the ``player'' and for the ``adversary'' as defined in the introduction.
\end{theorem}

Note that it should not come as a surprise that {\sc exp2} (with John's exploration) is suboptimal, since even in the full information case the basic {\sc exp2} strategy was provably suboptimal, see Theorem \ref{lb:expinfty}. We conjecture that the correct order of magnitude for the minimax regret in the bandit case is $m \sqrt{d n}$, as the above lower bound suggests.

A promising approach to resolve this conjecture is to consider again the {\sc osmd} approach. However we believe that in the bandit case, one has to consider Legendre functions with non-diagonal Hessian (on the contrary to the Legendre functions considered so far in this paper). Abernethy, Hazan, and Rakhlin \cite{AHR08} propose to use a self-concordant barrier function for the polytope $\conv(\cA)$. Then they randomly perturb the point $x_t$ given by {\sc osmd} using the eigenstructure of the Hessian. This approach leads to a regret upper bound of order $m d\sqrt{\theta n \log n}$ for $\theta>0$ when $\Conv(\cA)$ admits a $\theta$-self-concordant barrier function. Unfortunately, even when there exists a $O(1)$-self concordant barrier, this bound is still larger than the conjectured optimal bound by a factor $\sqrt{d}$. In fact, it was proved in \cite{BCK12}
that in some cases there exist better choices for the Legendre function and the perturbation than those described in \cite{AHR08}, even when there is a $O(1)$-self concordant function for the action set. How to generalize this approach to the polytopes involved in online combinatorial optimization is a challenging open problem.

\appendix

\section{Proof of Theorem \ref{lb:expinfty}.}
For the sake of simplicity, we assume that $d$ is a multiple of $4$ and that $n$ is even. We consider the following subset of the hypercube:
\begin{align*}
& \cA = \bigg\{a \in \{0,1\}^d : \sum_{i=1}^{d/2} a_i = d/4 \;\; \text{and} \;\; \\
& \bigg( a_{i} = 1, \forall i \in \{d/2+1; \ldots, d/2+d/4\}\bigg) \;\; \text{or} \;\; \bigg(a_i = 1, \forall i \in \{d/2+d/4+1, \ldots, d\}\bigg) \bigg\}.
\end{align*}
That is, choosing a point in $\cA$ corresponds to choosing a subset of $d/4$ elements among the first half of the coordinates, and choosing one of the two first disjoint
intervals of size $d/4$ in the second half of the coordinates. 

We prove that for any parameter $\eta$, there exists an adversary such that Exp2 (with 
parameter $\eta$) has a regret of at least $\frac{n d}{16} \tanh\big(\frac{\eta d}8\big)$, and that there exists another adversary such that its regret is at least
$\min\big(\frac{d \log 2}{12 \eta}, \frac{n d}{12}\big)$. 
As a consequence, we have
  \begin{align*}
  \sup R_n & \ge \max\bigg(\frac{n d}{16} \tanh\Big(\frac{\eta d}8\Big),\min\left(\frac{d \log 2}{12 \eta}, \frac{n d}{12}\right)\bigg) \\
  & \ge \min\bigg(\max\left(\frac{n d}{16} \tanh\Big(\frac{\eta d}8\Big),\frac{d \log 2}{12 \eta}\right), \frac{n d}{12}\bigg)
   \ge \min\Big( A , \frac{n d}{12}\Big),
  \end{align*}
with \begin{align*}
A&=\min_{\eta\in[0,+\infty)}\max\left(\frac{n d}{16} \tanh\Big(\frac{\eta d}8\Big),\frac{d \log 2}{12 \eta}\right)\\
& \ge \min\bigg( \min_{\eta d \ge 8} \frac{n d}{16} \tanh\Big(\frac{\eta d}8\Big),
  \min_{\eta d < 8} \max\left(\frac{n d}{16} \tanh\Big(\frac{\eta d}8\Big),\frac{d \log 2}{12 \eta}\right)\bigg)\\
  & \ge \min\bigg( \frac{n d}{16} \tanh(1),
  \min_{\eta d < 8} \max\left(\frac{n d}{16} \frac{\eta d}8 \tanh(1),\frac{d \log 2}{12 \eta}\right)\bigg)\\
  & \ge \min\Bigg( \frac{n d}{16} \tanh(1),
   \sqrt{\frac{n d^3 \log 2 \cdot \tanh(1)}{128\cdot 12}}\Bigg) \ge \min\big( 0.04\, n d, 0.01 \, d^{3/2} \sqrt{n}\big)~,
\end{align*}
where we used the fact that $\tanh$ is concave and increasing on $\R_+$.
As $n\ge d$, this implies the stated lower bound.

First we prove the lower bound $\frac{n d}{16} \tanh\big(\frac{\eta d}8\big)$. Define the following adversary:
$$z_t(i) = \left\{
\begin{array}{ccc}
1 & \text{if} & i \in \{d/2+1; \ldots, d/2+d/4\} \;\; \text{and} \;\; t \;\; \text{odd},\\
1 & \text{if} & i \in \{d/2+d/4+1, \ldots, d\} \;\; \text{and} \;\; t \;\; \text{even},\\
0 & \text{otherwise}. &
\end{array}
\right.$$
This adversary always puts a zero loss on the first half of the coordinates, and alternates between a loss of $d/4$ for choosing the first interval (in the
second half of the coordinates) and the second interval. At the beginning of odd rounds, any vertex $a \in \cA$ has the same cumulative loss and thus Exp2 picks its expert uniformly at random, 
which yields an expected cumulative loss equal to $n d / 16$. On the other hand, at even rounds the probability distribution to select the vertex $a \in \cA$ is always the same. 
More precisely,
the probability of selecting a vertex which contains the interval $\{d/2+d/4+1,\ldots,d\}$ (i.e, the interval with a $d/4$ loss at this round) is exactly $\frac{1}{1+\exp(-\eta d /4)}$.
This adds an expected cumulative loss equal to $\frac{n d}{8} \frac{1}{1+\exp(-\eta d /4)}$. Finally, note that the loss of any fixed vertex is $n d / 8$. Thus, we obtain
\begin{align*}
R_n = \frac{n d}{16} + \frac{n d}{8} \frac{1}{1+\exp(-\eta d /4)} - \frac{n d}{8}  = \frac{n d}{16} \tanh\Big(\frac{\eta d}8\Big).
\end{align*}

It remains to show a lower bound proportional to $1/\eta$. 
To this end, we consider a different adversary defined by
$$z_t(i) = \left\{
\begin{array}{ccc}
1-\epsilon & \text{if} & i \leq d/4, \\
1 & \text{if} & i \in \{d/4+1, \ldots, d/2\}, \\
0 & \text{otherwise}, &
\end{array}
\right.$$
for some fixed $\epsilon>0$.

Note that against this adversary the choice of the interval (in the
second half of the components) does not matter.  Moreover, by symmetry,
the weight of any coordinate in $\{d/4+1,\ldots,d/2\}$ is the same (at
any round). Finally, note that this weight is decreasing with
$t$. Thus, we have the following identities (in the big sums 
$i$ represents the number of components selected in the
first $d/4$ components):
\begin{align*}
R_n & = \frac{n \epsilon d}{4} \frac{\sum_{a \in \cA : a_{d/2}=1} \exp(- \eta n z_1^T a)}{\sum_{a \in \cA} \exp(- \eta n z_1^T a)} \\
& = \frac{n \epsilon d}{4}  \frac{\sum_{i=0}^{d/4-1} \binom{d/4}{i} \binom{d/4-1}{d/4-i-1} \exp(- \eta (n d / 4 - i n \epsilon))}
{\sum_{i=0}^{d/4} \binom{d/4}{i} \binom{d/4}{d/4-i} \exp(- \eta (n d / 4 - i n \epsilon))} \\
& = \frac{n \epsilon d}{4}  \frac{\sum_{i=0}^{d/4-1} \binom{d/4}{i} \binom{d/4-1}{d/4-i-1} \exp(\eta i n \epsilon)}
{\sum_{i=0}^{d/4} \binom{d/4}{i} \binom{d/4}{d/4-i}\exp(\eta i n \epsilon)} \\
& = \frac{n \epsilon d}{4} \frac{\sum_{i=0}^{d/4-1} \big(1 - \frac{4i}d\big) \binom{d/4}{i} \binom{d/4}{d/4-i} \exp(\eta i n \epsilon)}
{\sum_{i=0}^{d/4} \binom{d/4}{i} \binom{d/4}{d/4-i} \exp(\eta i n \epsilon)}
\end{align*}
where we used 
$\binom{d/4-1}{d/4-i-1} = \big(1 - \frac{4i}d\big) \binom{d/4}{d/4-i}$ in the last equality. Thus, taking 
$\epsilon = \min\big(\frac{\log 2}{\eta n}, 1\big) $ yields
$$R_n  \geq \min\left(\frac{d \log 2}{4 \eta}, \frac{n d}{4}\right) \frac{\sum_{i=0}^{d/4-1} \big(1-  \frac{4i}d\big) \binom{d/4}{i}^2 
  \min(2, \exp(\eta n))^i}
  {\sum_{i=0}^{d/4} \binom{d/4}{i}^2 \min(2, \exp(\eta n))^i}
  \geq \min\left(\frac{d \log 2}{12 \eta}, \frac{n d}{12}\right),$$
where the last inequality follows from Lemma \ref{lem:tech1} in the appendix.
This concludes the proof of the lower bound.

\section{Proof of Theorem \ref{thm:minimaxlower}}
The structure of the proof is similar to that of
\cite[Theorem 30]{AB10}, which deals with the simple case where $m=1$. 
The main important conceptual difference is contained in Lemma \ref{lem:KLbinomials}, which is at the heart of this new proof. The main argument follows
the line of standard lower bounds for bandit problems, see, e.g., \cite{CL06}:
The worst-case regret is bounded from below by by taking an average over a conveniently
chosen class of strategies of the adversary. Then, by Pinsker's inequality,
the problem is reduced to computing the Kullback-Leibler divergence 
of certain distributions. The main technical argument, given 
in Lemma \ref{lem:KLbinomials}, is for proving manageable bounds for
the relevant Kullback-Leibler divergence.

For the sake of simplifying notation, we assume that $d$ is a multiple of $m$, and we identify $\{0,1\}^d$ with the set of $m\times (d/m)$ binary matrices
$\{0,1\}^{m \times \frac{d}{m}}$. We consider the following set of actions:
$$\cA = \{a \in \{0,1\}^{m \times \frac{d}{m}} : \forall i \in \{1,\ldots, m\}, \sum_{j=1}^{d/m} a(i,j) =1 \}.$$ 
In other words, the player is playing in parallel $m$ finite games with $d/m$ actions. 

From step 1 to 3 we restrict our attention to the case of deterministic strategies for the player, and we show how to extend the results to arbitrary strategies in step 4.

\medskip
\noindent {\em First step: definitions.}

We denote by $I_{i,t} \in \{1, \ldots, m\}$ the random variable such that $a_t(i, I_{i,t}) =1$. That is, $I_{i,t}$ is the action chosen at time $t$ in the $i^{th}$ game. Moreover, let $\tau$ be drawn uniformly at random from $\{1,\ldots,n\}$.

In this proof we consider random adversaries indexed by $\cA$. More precisely, for $\alpha \in \cA$, we define the {\em $\alpha$-adversary} as follows:
For any $t \in \{1,\ldots,n\}$, $z_t(i,j)$ is drawn from a Bernoulli distribution with parameter $\frac12 - \epsilon \alpha(i,j)$. In other words, against adversary $\alpha$, in the $i^{th}$ game, the action $j$ such that $\alpha(i,j)=1$ has a loss slightly smaller (in expectation) than the other actions. We denote by $\E_{\alpha}$ integration with respect to the loss
generation process of the $\alpha$-adversary. We write $\P_{i, \alpha}$ for the probability distribution  of $\alpha(i, I_{i,\tau})$ when the player faces the $\alpha$-adversary. Note that we have $\P_{i, \alpha}(1) = \E_{\alpha} \frac{1}{n} \sum_{t=1}^n \IND_{\alpha(i,I_{i,t})=1}$, 
hence, against the $\alpha$-adversary, we have
$$\oR_n = \E_{\alpha} \sum_{t=1}^n \sum_{i=1}^{m} \epsilon \IND_{\alpha(i,I_{i,t}) \neq 1} = n \epsilon \sum_{i=1}^{m} \left(1 - \P_{i, \alpha}(1) \right),$$
which implies (since the maximum is larger than the mean)
\begin{equation} \label{eq:firststepmm}
\max_{\alpha \in \cA} \oR_n \geq n \epsilon \sum_{i=1}^{m} \left(1 - \frac{1}{(d/m)^m} \sum_{\alpha \in \cA} \P_{i, \alpha}(1) \right).
\end{equation}

\noindent {\em Second step: information inequality.}

Let $\P_{-i,\alpha}$ be the probability distribution of $\alpha(i,I_{i,\tau})$ against the adversary which plays like the $\alpha$-adversary except that in the $i^{th}$ game, the losses of all coordinates are drawn from a Bernoulli distribution of parameter $1/2$. We call it the {\em $(-i,\alpha)$-adversary} and we denote by $\E_{(-i,\alpha)}$  integration with respect to its loss
generation process. By Pinsker's inequality,
$$\P_{i, \alpha}(1) \leq \P_{- i, \alpha}(1)  + \sqrt{\frac{1}{2} \K(\P_{- i, \alpha},\P_{i, \alpha})}~,$$
where $\K$ denotes the Kullback-Leibler divergence. 
Moreover, note that by symmetry of the adversaries $(-i,\alpha)$,
\begin{eqnarray} 
\frac{1}{(d/m)^m} \sum_{\alpha \in \cA} \P_{-i, \alpha}(1) & = & \frac{1}{(d/m)^m} \sum_{\alpha \in \cA} \E_{(-i,\alpha)} \alpha(i, I_{i,\tau}) \notag \\
& = & \frac{1}{(d/m)^m} \sum_{\beta \in \cA}  \frac{1}{d/m} \sum_{\alpha: (-i,\alpha) = (-i, \beta)} \E_{(-i,\alpha)} \alpha(i, I_{i,\tau}) \notag \\
& = & \frac{1}{(d/m)^m} \sum_{\beta \in \cA}  \frac{1}{d/m} \E_{(-i,\beta)} \sum_{\alpha: (-i,\alpha) = (-i, \beta)} \alpha(i, I_{i,\tau}) \notag \\
& = & \frac{1}{(d/m)^m} \sum_{\beta \in \cA}  \frac{1}{d/m} \notag \\
& = & \frac{m}{d}, \label{eq:sym}
\end{eqnarray}
and thus, thanks to the concavity of the square root,
\begin{equation} \label{eq:pinsk}
\frac{1}{(d/m)^m} \sum_{\alpha \in \cA} \P_{i, \alpha}(1) \leq \frac{m}{d} + \sqrt{\frac{1}{2 (d/m)^m} \sum_{\alpha \in \cA} \K(\P_{- i, \alpha},\P_{i, \alpha})}.
\end{equation}

\noindent {\em Third step: computation of $\K(\P_{- i, \alpha},\P_{i, \alpha})$
with the chain rule.} 

Note that since the forecaster is deterministic, the sequence of
observed losses (up to time $n$) $W_n \in \{0, \ldots, m\}^n$
uniquely determines the empirical distribution of plays, and,
in particular, the probability distribution of $\alpha(i, I_{i,\tau})$ conditionally to $W_n$ is the same
for any adversary. Thus, if we denote by $\P_{\alpha}^n$ (respectively
$\P_{-i,\alpha}^n$) the probability distribution of $W_n$ when the forecaster plays against
the $\alpha$-adversary (respectively the $(-i,\alpha)$-adversary),
then one can easily prove that
$
\K(\P_{- i, \alpha},\P_{i, \alpha}) \leq \K(\P_{-i,\alpha}^n, \P_{\alpha}^n)
$. 
Now we use the chain rule for Kullback-Leibler divergence
iteratively to introduce the probability distributions $\P^t_{\alpha}$ of the observed
losses $W_t$ up to time $t$. More precisely, we have,
\begin{align*}
& \K(\P_{-i,\alpha}^n, \P_{\alpha}^n) \\
& =  \K(\P_{-i,\alpha}^1, \P_{\alpha}^1) + \sum_{t=2}^n \sum_{w_{t-1} \in \{0,\ldots,m\}^{t-1}} \P_{-i,\alpha}^{t-1}(w_{t-1}) \K(\P_{-i,\alpha}^t(. | w_{t-1}),\P_{\alpha}^t(. | w_{t-1})) \\
& =  \K\left(\cB_{\emptyset}, \cB_{\emptyset}'\right) \IND_{\alpha(i,I_{i,1})=1} + \sum_{t=2}^n \sum_{w_{t-1} : \alpha(i,I_{i,1})=1} \P_{-i,\alpha}^{t-1}(w_{t-1}) \K\left(\cB_{w_{t-1}}, \cB_{w_{t-1}}'\right),
\end{align*}
where $\cB_{w_{t-1}}$ and $\cB_{w_{t-1}}'$ are sums of $m$ Bernoulli distributions with parameters in $\{1/2,1/2-\epsilon\}$ and such that the number of Bernoullis with parameter $1/2$ in $\cB_{w_{t-1}}$ is equal to the number of Bernoullis with parameter $1/2$ in $\cB_{w_{t-1}}'$ plus one. 
Now using Lemma \ref{lem:KLbinomials} (see below) we obtain,
$$\K\left(\cB_{w_{t-1}}, \cB_{w_{t-1}}'\right) \leq  \frac{8 \; \epsilon^2}{(1-4 \epsilon^2) m}  .$$
In particular, this gives
$$\K(\P_{-i,\alpha}^n, \P_{\alpha}^n) \leq \frac{8 \; \epsilon^2}{(1-4 \epsilon^2) m} \E_{-i,\alpha} \sum_{t=1}^n \IND_{\alpha(i,I_{i,t})=1} = \frac{8 \; \epsilon^2 n}{(1 - 4 \epsilon^2) m} \P_{-i,\alpha}(1) .$$
Summing and plugging this into \eqref{eq:pinsk} we obtain  (again thanks to \eqref{eq:sym}), for $\epsilon \leq \frac{1}{\sqrt{8}}$,
$$\frac{1}{(d/m)^m} \sum_{\alpha \in \cA} \P_{i, \alpha}(1) \leq \frac{m}{d} + \epsilon \sqrt{\frac{8 n}{d}} .$$
To conclude the proof of \eqref{eq:LBbanditcombi} for deterministic players one needs to plug  
this last equation in \eqref{eq:firststepmm} along with straightforward computations.

\noindent {\em Fourth step: Fubini's theorem to handle non-deterministic players.}

Consider now a randomized player, and let $\E_{rand}$ denote the expectation with respect to the randomization of the player. Then one has (thanks to Fubini's theorem),
$$
\frac{1}{(d/m)^m} \sum_{\alpha \in \cA} \E \sum_{t=1}^n (a_t^T z_t- \alpha^T z) = \E_{rand} \frac{1}{(d/m)^m}  \sum_{\alpha \in \cA} \E_{\alpha} \sum_{t=1}^n (a_t^T z_t- \alpha^T z) .$$
Now note that if we fix the realization of the forecaster's randomization then the results of the previous steps apply and, in particular, one can lower bound $\frac{1}{(d/m)^m}  \sum_{\alpha \in \cA} \E_{\alpha} \sum_{t=1}^n (a_t^T z_t- \alpha^T z)$ as before (note that $\alpha$ is the optimal action in expectation against the $\alpha$-adversary).

\section{Technical lemmas.}

\begin{lemma} \label{lem:tech1}
For any $k \in \N^*,$ for any $1\le c \le 2$, we have  
$$\frac{\sum_{i=0}^{k} (1- i/k) \binom{k}{i}^2 c^i}
{\sum_{i=0}^{k} \binom{k}{i}^2 c^i} \geq 1/3.$$
\end{lemma}

\begin{proof}
Let $f(c)$ denote the expression on the left-hand side of the inequality. Introduce
the random variable $X$, which is equal to $i\in\{0,\ldots,k\}$ with
probability $\binom{k}{i}^2 c^i\big/ \sum_{j=0}^k \binom{k}{j}^2 c^j$.
We have 
$f'(c)= \frac1c \E[X(1-X/k)] - \frac1c \E(X)\E(1-X/k) = -
\frac{1}{c k} \Var X \le 0.$ 
So the function $f$ is decreasing on $[1,2]$,
and therefore it suffices to consider $c=2$.  Numerator and denominator of the
left-hand side differ only by the  factor $1-i/k$.  A lower
bound for the left-hand side can thus be obtained by showing that
the terms for $i$ close to $k$ are not essential to the value of the
denominator.  To prove this, we may use Stirling's formula
which implies that for any $k\ge 2$ and $i\in[1,k-1]$,
  \begin{align*}
  \Big(\frac{k}i\Big)^i \Big(\frac{k}{k-i}\Big)^{k-i} \frac{\sqrt{k}}{\sqrt{2\pi i(k-i)}} e^{-1/6} < \binom{k}{i} 
    < \Big(\frac{k}i\Big)^i \Big(\frac{k}{k-i}\Big)^{k-i} \frac{\sqrt{k}}{\sqrt{2\pi i(k-i)}} e^{1/12},
  \end{align*}
hence
  \begin{align*}
  \Big(\frac{k}i\Big)^{2i} \Big(\frac{k}{k-i}\Big)^{2(k-i)} \frac{k e^{-1/3}}{{2\pi i(k-i)}}  < \binom{k}{i}^2
    < \Big(\frac{k}i\Big)^{2i} \Big(\frac{k}{k-i}\Big)^{2(k-i)} \frac{k e^{1/6}}{{2\pi i}}~.  
  \end{align*}
Introduce $\lam=i/k$ and $\chi(\lam)=\frac{2^\lam}{\lam^{2\lam}(1-\lam)^{2(1-\lam)}}$.
We have 
  \begin{align} \label{eq:stirlam}
  [\chi(\lam)]^k \frac{2 e^{-1/3}}{\pi k}  
    < \binom{k}{i}^2 2^i
    < [\chi(\lam)]^k \frac{e^{1/6}}{2\pi \lam}.
  \end{align}
Lemma \ref{lem:tech1} can be numerically verified for $k\le 10^6$.
We now consider $k>10^6$. 
For $\lam\ge 0.666$, since the function $\chi$ can be shown to be decreasing on $[0.666,1]$, the inequality
  $\binom{k}{i}^2 2^i < [\chi(0.666)]^k \frac{e^{1/6}}{2\times 0.666 \times  \pi}$ holds.
We have $\chi(0.657)/\chi(0.666)>1.0002$. Consequently, for 
$k>10^6$, we have $[\chi(0.666)]^k < 0.001 \times [\chi(0.657)]^k/k^2$.
So for $\lam\ge 0.666$ and $k>10^6$, we have
  \begin{align} 
  \binom{k}{i}^2 2^i < 0.001 \times [\chi(0.657)]^k \frac{e^{1/6}}{2\pi\times0.666 \times k^2}
  & <[\chi(0.657)]^k  \frac{2 e^{-1/3}}{1000 \pi k^2}  \notag  \\
  & = \min_{\lam\in[0.656,0.657]}[\chi(\lam)]^k \frac{2 e^{-1/3}}{1000 \pi k^2} \notag \\
  & < \frac1{1000k}\max_{i\in\{1,\dots,k-1\}\cap[0,0.666k)} \binom{k}{i}^2 2^i~,  \label{eq:cuta}
  \end{align}
where the last inequality comes from \eqref{eq:stirlam} and the fact that there exists $i\in\{1,\dots,k-1\}$
such that $i/k\in[0.656,0.657]$. Inequality \eqref{eq:cuta} implies that 
for any $i\in\{1,\dots,k\}$, we have
  $$
  \sum_{0.666 k\le i \le k}\binom{k}{i}^2 2^i 
  < \frac1{1000}\max_{i\in\{1,\dots,k-1\}\cap[0,0.666k)} \binom{k}{i}^2 2^i
  < \frac1{1000}\sum_{0\le i < 0.666 k}\binom{k}{i}^2 2^i .
  $$
To conclude, introducing $A=\sum_{0\le i <0.666 k}\binom{k}{i}^2 2^i$, we have
  \begin{align*}
  \frac{\sum_{i=0}^{k} (1- i/k) \binom{k}{i}^2 2^i}
{\sum_{i=0}^{k} \binom{k}{i} \binom{k}{k-i} 2^i} >
  \frac{(1-0.666) A}{A+0.001A} \ge \frac13.
  \end{align*}
\end{proof}

\begin{lemma} \label{lem:KLbinomials}
Let $\ell$ and $n$ be integers with $\frac12\le \frac{n}2\le \ell\le n$.
Let $p,p',q,p_1,\dots,p_n$ be real numbers in $(0,1)$ with $q\in\{p,p'\}$, $p_1=\cdots=p_\ell=q$ and 
$p_{\ell+1}=\cdots=p_n$.
Let $\cB$ (resp. $\cB'$) be the sum of $n+1$ independent Bernoulli distributions with parameters
$p,p_1,\dots,p_n$ (resp. $p',p_1,\dots,p_n$). We have
$$\KL(\cB, \cB') \le \frac{2(p'-p)^2}{(1-p')(n+2)q}.$$
\end{lemma}

\begin{proof}
Let $Z,Z',Z_1,\dots,Z_n$ be independent Bernoulli distributions with parameters $p,p',p_1,\dots,p_n$.
Define $S=\sum_{i=1}^\ell Z_i$, $T=\sum_{i=\ell+1}^n Z_i$ and $V=Z+S$.
By a slight and usual abuse of notation, we use $\KL$ to denote Kullback-Leibler
divergence of both probability distributions and random variables.
Then we may write (the inequality is an easy consequence of the chain rule for Kullback-Leibler divergence)
  \begin{align*}
  \KL(\cB, \cB') & = \KL\big((Z+S)+T,(Z'+S)+T\big) \\
  & \le \KL\big((Z+S,T),(Z'+S,T)\big) \\
  & = \KL\big(Z+S,Z'+S\big).
  \end{align*}
Let $s_k=\P(S=k)$ for $k=-1,0,\dots,\ell+1$.
Using the equalities 
  \begin{multline*}
  s_k = \binom{\ell}{k} q^k(1-q)^{\ell-k}
  = \frac{q}{1-q} \frac{\ell-k+1}k \binom{\ell}{k-1} q^{k-1}(1-q)^{\ell-k+1}
  = \frac{q}{1-q} \frac{\ell-k+1}k s_{k-1},
  \end{multline*}
which holds for $1 \le k\le \ell+1$, we obtain
  \begin{align}
  \KL(Z+S,Z'+S) & = \sum_{k=0}^{\ell+1} \P(V=k) 
    \log\bigg(\frac{\P(Z+S=k)}{\P(Z'+S=k)}\bigg) \notag\\
  & = \sum_{k=0}^{\ell+1} \P(V=k)
  \log \bigg(\frac{p s_{k-1}+(1-p) s_{k}}{p' s_{k-1}+(1-p') s_{k}}\bigg) \notag\\
  & = \sum_{k=0}^{\ell+1} \P(V=k)
    \log \bigg(\frac{p \frac{1-q}{q} k+(1-p)(\ell-k+1) }
    {p'\frac{1-q}{q} k+(1-p')(\ell-k+1) }\bigg) \notag\\
  & = \E
    \log \bigg(\frac{(p-q)V+(1-p)q(\ell+1) }
    {(p'-q)V+(1-p')q(\ell+1) }\bigg). \label{eq:klgen}
  \end{align}
\noindent {\em First case: $q=p'$.}
\newline
By Jensen's inequality, using that $\E V=p'(\ell+1)+p-p'$ in this case, we get
  \begin{align*}
  \KL(Z+S,Z'+S) & \le
    \log \bigg(\frac{(p-p')\E( V )+(1-p)p'(\ell+1) }
    {(1-p')p'(\ell+1) }\bigg) \\
  & = \log \bigg(\frac{(p-p')^2+(1-p')p'(\ell+1) }
    {(1-p')p'(\ell+1) }\bigg) \\
  & = \log \bigg(1+\frac{(p-p')^2}
    {(1-p')p'(\ell+1) }\bigg) 
    \le \frac{(p-p')^2}
    {(1-p')p'(\ell+1) }~.
  \end{align*}
\noindent {\em Second case: $q=p$.}
\newline
In this case, $V$ is a binomial distribution with parameters $\ell+1$ and $p$. 
From \eqref{eq:klgen}, we~have
  \begin{align}
  \KL(Z+S,Z'+S) & \le - \E
    \log \bigg(\frac{(p'-p)V+(1-p')p(\ell+1) }{(1-p)p(\ell+1) }
    \bigg)  \notag \\
   & \le - \E \log \bigg(1+\frac{(p'-p)(V-\E V)}{(1-p)p(\ell+1) }
    \bigg) . \label{eq:qp}
  \end{align}  
To conclude, we will use the following lemma.
\begin{lemma} \label{lem:log4}
The following inequality holds for any $x\ge x_0$ with $x_0\in(0,1)$:
  $$
  -\log(x) \le -(x-1)+\frac{(x-1)^2}{2x_0}.
  $$
\end{lemma}
\begin{proof}
Introduce $f(x)=-(x-1)+\frac{(x-1)^2}{2x_0}+\log(x)$.
We have $f'(x)=-1 + \frac{x-1}{x_0} +\frac1x$,
and $f''(x)=\frac1{x_0} -\frac1{x^2}$.
From $f'(x_0)=0$, we get that $f'$ is negative on $(x_0,1)$
and positive on $(1,+\infty)$. This leads to
$f$ nonnegative on $[x_0,+\infty)$.
\end{proof}
  
Finally, from Lemma \ref{lem:log4} and \eqref{eq:qp}, using
$x_0=\frac{1-p'}{1-p}$, we obtain
  \begin{align*}
  \KL(Z+S,Z'+S) & \le \bigg(\frac{p'-p}{(1-p)p(\ell+1)}\bigg)^2 
    \frac{\E[(V-\E V)^2]}{2x_0}\\
   & = \bigg(\frac{p'-p}{(1-p)p(\ell+1)}\bigg)^2 
    \frac{(\ell+1)p(1-p)^2}{2(1-p')}\\
   & =  \frac{(p'-p)^2}{2(1-p')(\ell+1)p}.
  \end{align*}  
\end{proof}

\section*{Acknowledgements}
G.\ Lugosi is supported by the Spanish Ministry of Science and Technology grant
MTM2009-09063 and PASCAL2 Network of Excellence under EC grant
no.\ 216886.

\bibliographystyle{amsplain}
\bibliography{newbib}

\end{document}